\documentclass[11pt,draftclsnofoot,onecolumn]{IEEEtran}

\usepackage{authblk}
\author[1]{Henry Kenlay}
\author[2]{Dorina Thanou}
\author[1]{Xiaowen Dong}
\affil[1]{University of Oxford}
\affil[2]{Swiss Data Science Center}

\usepackage{amsmath,graphicx}
\usepackage{amssymb} 
\usepackage{mathtools} 
\usepackage{amsthm} 
\usepackage{subcaption} 
\usepackage{tikz-network} 

\newcommand{\A}{\mathbf{A}} 
\newcommand{\B}{\mathbf{B}} 
\newcommand{\C}{\mathbf{C}} 
\newcommand{\E}{\mathbf{E}} 
\newcommand{\I}{\mathbf{I}} 
\newcommand{\D}{\mathbf{D}} 
\newcommand{\U}{\mathbf{U}} 
\newcommand{\R}{\mathbb{R}} 
\newcommand{\G}{\mathcal{G}} 
\newcommand{\V}{\mathcal{V}} 
\newcommand{\GSO}{\mathbf{\Delta}} 
\newcommand{\x}{\mathbf{x}} 
\newcommand{\X}{\mathbf{X}} 
\newcommand{\equaldef}{\overset{\underset{\mathrm{def}}{}}{=}} 
\DeclareMathOperator{\softmax}{softmax}
\DeclareMathOperator{\diag}{diag}

\DeclarePairedDelimiter\norm{\lVert}{\rVert} 
\DeclarePairedDelimiter\abs{\lvert}{\rvert} 

\makeatletter
\let\oldabs\abs
\def\abs{\@ifstar{\oldabs}{\oldabs*}}
\let\oldnorm\norm
\def\norm{\@ifstar{\oldnorm}{\oldnorm*}}
\makeatother

\newtheorem{lemm}{Lemma}
\newtheorem{prop}{Proposition}

\newtheorem{corr}{Corollary}

\title{On the Stability of Graph Convolutional Neural Networks under Edge Rewiring}

\begin{document}

\maketitle

\begin{abstract}
Graph neural networks are experiencing a surge of popularity within the machine learning community due to their ability to adapt to non-Euclidean domains and instil inductive biases. Despite this, their stability, i.e., their robustness to small perturbations in the input, is not yet well understood. Although there exists some results showing the stability of graph neural networks, most take the form of an upper bound on the magnitude of change due to a perturbation in the graph topology. However, the change in the graph topology captured in existing bounds tend not to be expressed in terms of structural properties, limiting our understanding of the model robustness properties. In this work, we develop an interpretable upper bound elucidating that graph neural networks are stable to rewiring between high degree nodes. This bound and further research in bounds of similar type provide further understanding of the stability properties of graph neural networks.
\end{abstract}
\section{Introduction}

Recently, there has been an increasing amount of research on the use of machine learning models which operate on graph-structured data \cite{bronstein2017geometric, wu2020comprehensive}. Graphs encode pairwise interactions and can help model non-Euclidean data such as social networks and molecules as well as impart inductive biases \cite{battaglia2018relational}. Graph signal processing (GSP), for example, generalises traditional signal processing to network domains allowing us to transfer many of the existing tools such as filtering and sampling \cite{emerging, dong2020graph}. GSP also allows us to generalise convolutions providing a framework to design graph convolutional neural networks (GCNNs) where the convolutional layers are filter banks of spectral graph filters \cite{defferrard2016convolutional, levie2018cayleynets, bianchi2019graph}. To date, much of the research has focused on the predictive performance of these models with far fewer studies looking at their theoretical properties.

Like their Euclidean counterparts, GCNNs are susceptible to adversarial attacks \cite{zugner2018adversarial, dai2018adversarial, jin2020adversarial}. An adversarial attack is a small but targeted perturbation of the input which causes large changes in the output \cite{szegedy2013intriguing}. In the case of GCNNs, the modification of a few edges in the input graph can change significantly the learned node representations so that the prediction in the downstream task, typically node or graph classification, switches from a correct to incorrect prediction. One approach to understanding the vulnerability of these models to adversarial attacks is to consider their robustness against perturbation. Robustness is an important property for reliable deployment of machine learning models in the real world, especially in domains where adversaries are common (e.g., on the web). 

Existing literature has shown that large functional classes of spectral graph filters, a key component of GCNNs, are bounded by the magnitude of the change in the input graph. The main drawback of these approaches is that the bounds involved in quantifying stability to specific changes in the topology do not have natural structural interpretations. For example, if we delete just a single edge, it is unclear how loose the bound on the output change will be even if we know the statistical properties about the edge and endpoint nodes.

Our main goal is to extend bounds found in the existing literature so that they have an interpretation in terms of the structural properties of the existing and perturbed graphs. To do this, we focus on polynomial graph filters, and build on our previous work \cite{polynomialstability}, by providing a new bound that is tighter and generalises to the family of normalised augmented adjacency matrices. We then bound the change in normalised augmented adjacency matrix by considering the largest change around each node where the change admits a structural interpretation. As a specific example, we consider perturbations that do not modify the degree distribution of the graph (i.e., double edge rewiring, illustrated in Fig.~\ref{fig:rewiring}). We then discuss under which scenarios, deleting and adding edges between nodes, guarantee the filter to be robust to perturbations. 
To demonstrate how these theoretical results can be combined, we bound the change in the representations learned by two well-known GCNN architectures. Specifically, we consider simple graph convolutional networks \cite{SGCN} and multilayered graph convolutional network \cite{GCN} for scenarios where the graph topology is perturbed using edge rewiring. 

\section{Related work}
One of the earlier works on stability was by Levie et al \cite{levie1}. In this work, the authors give an upper bound of change which grows linearly in the distance between the graph shift operators before and after perturbation for a large class of spectral graph filters. 
The bound is based on analysis in the Cayley smoothness space. Related to this work, $\cite{polynomialstability}$ also provides a linear bound for polynomial spectral filters, where the bound is based on Taylor expansion for matrix functions. In Sec.~\ref{sec:polynomial} we give a tighter bound for polynomial spectral filters via a much simpler proof.

Gama et al. prove that a class of spectral graph filters are stable to changes in the graph topology \cite{gama2019stability}. Furthermore, GCNNs using filters in this class as filter banks and ReLU nonlinearities are also stable. The main difference between this work and the work presented in this paper is on how the magnitude of perturbation is measured. In \cite{gama2019stability} the authors posit that simple additive error as a measure does not reflect the fact that 
the distance between isomorphic graphs can be non-zero. To address this concern they consider a relative measure of perturbation and consider all node permutations for the perturbed graph. We believe that the additive error approach of \cite{ polynomialstability, levie1} and this work, which does not consider node permutations, and the approach of \cite{gama2019stability} are both useful. 
For example, if the graph represents a polygon mesh, then the node labelling in the perturbed graph does not have any meaningful interpretation, as they are just used to construct matrix representations. In this case, considering permutations is appropriate. However, in a social network, the node labelling will typically correspond to the identification of a user, in which case it makes sense to consider the labelling as fixed between the original and perturbed graph.

Beyond the notion of stability adopted in the aforementioned studies, there are efforts to quantify alternative notions of stability for graph-based models. An input to a classification model is certifiably robust if for any perturbation, under a given perturbation model, the predicted label will not change. Methods for generating robustness certificates for nodes in semi-supervised learning tasks have recently been proposed for graph neural networks \cite{NEURIPS2019_e2f374c3, zugner2019certifiable}. 
One can also measure the stability of graph-based models by considering the graph topology and signal as random variables and considering the statistical properties of the model. For example, \cite{isufi2017filtering} proves that in stochastic time-evolving graphs, the output of the filter behaves the same as the deterministic filter in expectation. In \cite{ceci2018robust}, the authors show how to deal with uncertainties in the graph topology to approximate the original filtering process. In \cite{9054424} stochastic graph neural networks are proposed to account for training on a stochastic graph. Unlike the existing approaches outlined in this section, to the best of our knowledge, our work is one of the first to provide sufficient conditions for stability that come with a structural interpretation.

\section{Problem formulation}
We consider unweighted and undirected graphs $\G=(\V, \mathcal{E})$ where $\V$ is a finite set of nodes and $\mathcal{E}$ is the edge set. The adjacency matrix $\A$ encodes connections between nodes with $\A_{uv}=1$ if there is an edge between nodes $u$ and $v$ and zero otherwise. The degree $d_u$ of a node $u$ is the number of nodes that $u$ is connected to. A graph signal is a function $x:\V \rightarrow \mathbb{R}$ that assigns a scalar value to each node. By fixing a labelling of the nodes we can represent this as a vector $\x \in \R^n$ where $\x_i$ is the function value of node $i$ and $n=\abs{\V}$ is the number of nodes in the graph. We will be concerned with the normalised augmented adjacency matrix $\GSO_\gamma = \D_\gamma^{-1/2} \A_\gamma \D_\gamma^{-1/2}$ where $\A_\gamma = \A + \gamma \I$ and $\D_\gamma = \D + \gamma \I$ with $\gamma \geq 0$. When $\gamma=0$ this matrix is the normalised adjacency matrix.  For simplicity, we drop the $\gamma$ subscript when the context is clear or the specific value of $\gamma$ is unimportant. We can write $\GSO$ in terms of its entries as
\begin{equation}
\label{eq:aamentries}
\GSO_{uv} = \begin{cases}
    \frac{\gamma}{d_u + \gamma}  &\textnormal{if }u = v \\
    \frac{1}{\sqrt{(d_u + \gamma)(d_v + \gamma)}}  &\textnormal{if }\A_{uv}=1 \textnormal{ and } u \not = v \\
    0  &\textnormal{otherwise} \\
    \end{cases}.
\end{equation}

Normalised augmented adjacency matrices admit an eigendecomposition $\GSO=\U \mathbf{\Lambda} \U^\intercal$ where the columns of $\U$ are the orthonormal eigenvectors, and $\mathbf{\Lambda} = \diag(\lambda_1, \ldots \lambda_n)$ where $\lambda_1 \leq, \ldots, \leq \lambda_n$ is the diagonal matrix of eigenvalues. The graph Fourier transform of a signal $\x$ is given by $\hat \x = \U^\intercal \x$, with the inverse graph Fourier transform defined as $\x = \U \hat \x$. With a notion of Fourier basis,  a spectral graph filter is defined as a function $g:\R \rightarrow \R$ which amplifies and attenuates specific frequencies. We can filter a signal by applying the filter to the graph shift operator (in this case $\GSO$) directly $$\U \diag(g(\lambda_1), \ldots g(\lambda_n))\U^\intercal \x = \U g(\mathbf{\Lambda})\U^\intercal \x = g(\GSO)\x.$$

In this work $\norm*{\cdot}_2$ represents the Euclidean norm when applied to vectors and the operator norm when applied to matrices. We will also consider the Frobenius norm $\norm*{\A}_F^2=\sum_{i,j}\A_{i,j}^2$, the matrix one norm $\norm*{\A}_1=\max_{i}\sum_{j}\abs{\A_{ij}}$ and the matrix infinity norm $\norm*{\A}_\infty=\max_{j}\sum_{i}\abs{\A_{ij}}$. The eigenvalues of the normalised augmented adjacency matrix lie in the interval $[-1, 1]$. Furthermore, $1$ is always an eigenvalue so $\norm*{\GSO}_2=1$.

Given a graph $\G$ and a perturbed graph $\G_p$ with normalised augmented adjacency matrices $\GSO$ and $\GSO_p$ respectively, our goal is twofold. First, we want to understand the stability of spectral graph filters in terms of structural perturbation, by providing bounds that quantify the change in the output. Second, we want to find sufficient conditions for this change to be small, by using interpretable structural properties of the graphs and the perturbation. We address the first goal in Sec.~\ref{sec:polynomial} and the second in Sec.~\ref{sec:rewiring}. Furthermore, we demonstrate how these theoretical contributions can be combined to give insight into the stability of GCNNs. In Sec.~\ref{sec:GNN}, we combine the results of Sec.~\ref{sec:polynomial} and Sec.~\ref{sec:rewiring} to show how we can provide sufficient conditions for the stability of two popular GCNN architectures.

\section{Stability of polynomial filters}
\label{sec:polynomial}
Our notion of stability is based on relative output distance which is bounded by the filter distance
\begin{equation*}
\frac{\norm{g(\GSO)\x - g(\GSO_p)\x}_2}{\norm{\x}_2} \leq \max_{\x\not = 0} \frac{\norm{g(\GSO)\x - g(\GSO_p)\x}_2}{\norm{\x}_2} \equaldef \norm{g(\GSO) - g(\GSO_p)}_2. 
\end{equation*}
In \cite{polynomialstability} we bounded the filter distance of polynomial filters by some constant times the error $\norm*{\E}_2$, where the constant depends on the filter and the error is the magnitude of the difference between normalised Laplacian matrices. We say that by satisfying this condition the filters are stable. Although the focus of this paper is not on tight bounds, we improve this bound by finding a smaller constant for polynomial filters. To do so we will use the following Lemma.
\begin{lemm}[Lemma 3, \cite{levie1}]
\label{lemma:levie}
Suppose $\mathbf{B}, \mathbf{D}, \mathbf{E} \in \mathbb{C}^{N \times N}$ are Hermitian matrices satisfying $\mathbf{B}=\mathbf{D}+\mathbf{E}$, and $\norm{\mathbf{B}}_2, \norm{\mathbf{D}}_2 \leq C$ for some $C > 0$. Then for every $l \geq 0$ 
\begin{equation*}
    \norm{\mathbf{B}^l-\mathbf{D}^l}_2 \leq lC^{l-1}\norm{\E}_2.
\end{equation*}
\end{lemm}
From here on in we will write $\E=\mathbf{\GSO_p}-\mathbf{\GSO}$ to be the difference of normalised augmented adjacency matrices between a graph $\G$ and a perturbed version of the graph $\G_p$. The smaller constant is given by the following proposition.
\begin{prop}
\label{prop:polystab}
Let $\GSO$ and $\GSO_p$ be the normalised augmented adjacency matrix for $\G$ and $\G_p$. Consider a polynomial graph filter $g_\theta(\lambda) = \sum_{k=0}^K \theta_k \lambda^k$. Then
\begin{equation*}
    \norm{g_\theta(\GSO) - g_\theta(\GSO_p)}_2 \leq \sum_{k=1}^K k\abs{\theta_k} \norm{\E}_2.
\end{equation*}
\end{prop}
\begin{proof}
Using the triangle inequality followed by an application of Lemma \ref{lemma:levie} with constant $C=1$ we get 
\begin{equation*}
\norm{g_\theta(\GSO) - g_\theta(\GSO_p)}_2 = \norm{\sum_{k=1}^K \theta_k (\GSO^k - \GSO_p^k)}_2 \leq \sum_{k=1}^K \abs{\theta_k} \norm{\GSO^k - \GSO_p^k}_2 \leq \sum_{k=1}^K k \abs{\theta_k} \norm{\E}_2. \qedhere
\end{equation*}
\end{proof}
In this work, we assume that the parameters of the model are fixed before and after perturbation. In the adversarial learning literature, an attack that modifies the input to cause large changes in the output whilst the model parameters are fixed is known as an evasion attack \cite{jin2020adversarial}. Robust models in the context of our work are those that are robust to evasion attacks with respect to the graph structure.
\section{Robustness to edge rewiring perturbations} 
\label{sec:rewiring}

In this section we bound the error term $\norm*{\E}_2$ by interpretable properties relating to the structural change. Consider a graph $\G$ which we perturb to arrive at $\G_p$. Our approach to upper bounding $\norm*{\E}_2$ relies on the inequality $\norm*{\E}_2^2 \leq \norm*{\E}_1 \norm*{\E}_\infty$ \cite[Section 6.3]{higham2002accuracy}. As $\E$ is symmetric $\norm*{\E}_1 = \norm*{\E}_\infty$ giving $\norm*{\E}_2 \leq \norm*{\E}_1$. There may exist strategies which give tighter bounds, but the benefit of this approach to bounding the error term is that $\norm*{\E}_1$ leads to an interpretation in the structural domain. For a matrix $\E$ we write $\E_u$ as the $u$th column of $\E$ so $\E_u^\intercal$ is the $u$th row. The row $\E_u^\intercal$  corresponds to the node $u$ in the graph. By definition $\norm*{\E}_1 = \max_{u \in \V} \norm*{\E_u^\intercal}_1$ where $\norm*{\E_u^\intercal}_1=\sum_v \abs{\E_{uv}}$ is the Manhattan norm of the row. Perturbations which cause small changes to  $\norm*{\E_u^\intercal}_1$ over all nodes $u$ guarantee small change in terms of $\norm*{\E}_2$. 

The focus of this section is on developing an interpretable upper bound on $\norm{\E}_2$. Before establishing an upper bound, we briefly mention the following lower bound to the error term \begin{equation*}
\max_{i,j} \abs{\E_{ij}} \equaldef \norm{\E}_{\max{}} \leq \norm{\E}_2.
\end{equation*}
This lower bound gives us sufficient conditions for large values of the error term. For example, deleting or adding an edge such that $\sqrt{(d_u+\gamma)(d_v+\gamma)}$ is small (e.g., if both degrees are small) will cause $\norm*{\E}_2$ to be large. In these cases, our bound will be loose and cannot provide guarantees about the filter robustness.

We aim to understand what type of perturbations spectral graph filters may be robust to. As a specific example consider perturbations that preserve the degree of the nodes. For this scenario we can consider how the entries change from $\GSO$ to $\GSO_p$ by considering Eq.~(\ref{eq:aamentries}) to write a closed-form for $\norm*{\E_u^\intercal}_1$. We will write $\mathcal{A}_u$ to be the set of edges added around a node $u$ and $\mathcal{D}_u$ to be the set of edges deleted around a node $u$. The diagonal entries of $\GSO$ and $\GSO_p$ remain unchanged, and each edge deletion flips the entry from zero to the second case of Eq.~(\ref{eq:aamentries}), whilst edge addition flips the entry the other way. This insight lets us write $\norm{\E_u^\intercal}_1$ in closed-form as 
\begin{align}
\label{eq:Eu}
\norm{\E_u^\intercal}_1 &=  \sum_{v \in \mathcal{D}_u} \frac{1}{\sqrt{(d_u+\gamma)(d_v+\gamma)}} + \sum_{v \in \mathcal{A}_u} \frac{1}{\sqrt{(d_u+\gamma)(d_v+\gamma)}} \nonumber \\
&= \frac{1}{\sqrt{d_u+\gamma}} \left( \sum_{v \in \mathcal{D}_u} \frac{1}{\sqrt{d_v+\gamma}} + \sum_{v \in \mathcal{A}_u} \frac{1}{\sqrt{d_v+\gamma}} \right)
\end{align}

One such perturbation is edge rewiring that preserves degree distribution. We define double edge rewiring as a function of two edges $(u, v)$ and $(u', v')$ such that $u$ is not connected to $u'$ or $v'$ and similarly $v$ is not connected to $u'$ and $v'$. The operation consists of deleting edges $(u, v)$ and $(u', v')$ and adding edges $(u, u')$ and $(v, v')$. This operation is depicted graphically in Fig.~\ref{fig:rewiring}.  Although the precise definition of rewiring in this work is slightly different, the idea of rewiring has been proposed as a strategy to make modifications imperceptible in the context of topological adversarial attacks \cite{ma2019attacking}. Approximately preserving the degree distribution has also been a criterion used to define imperceptibility \cite{zugner2018adversarial}. Beyond the adversarial attack literature, double edge rewiring has been used to model changes in a network where the capacity of a node is fixed and remains at full load such as in communication networks \cite{bienstock1994degree}.   
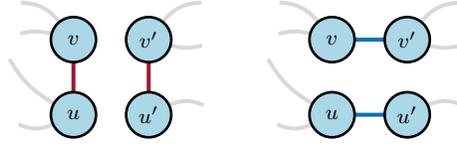
\begin{figure}
    \centering
    \begin{subfigure}{0.2\textwidth}
        \centering
        \begin{tikzpicture}
\Vertex[x=0,y=0, label=$u$]{u}
\Vertex[x=0,y=1, label=$v$]{v}
\Vertex[x=1,y=0, label=$u'$]{uprime}
\Vertex[x=1,y=1, label=$v'$]{vprime}
\Edge[RGB,color={162.56 ,  19.968,  47.104}](u)(v)
\Edge[RGB,color={162.56 ,  19.968,  47.104}](uprime)(vprime)
\Vertex[x=-1,y=1.5, Pseudo]{dummy13}
\Vertex[x=-1,y=1, Pseudo]{dummy12}
\Vertex[x=-1,y=0, Pseudo]{dummy11}
\Edge[bend=25, opacity=0.2](u)(dummy11)
\Edge[bend=15, opacity=0.2](u)(dummy12)
\Edge[bend=-15, opacity=0.2](v)(dummy12)
\Edge[bend=-25, opacity=0.2](v)(dummy13)
\Vertex[x=2,y=1.5, Pseudo]{dummy23}
\Vertex[x=2,y=1, Pseudo]{dummy22}
\Vertex[x=2,y=0, Pseudo]{dummy21}
\Edge[bend=25, opacity=0.2](uprime)(dummy21)
\Edge[bend=-15, opacity=0.2](vprime)(dummy22)
\Edge[bend=25, opacity=0.2](vprime)(dummy23)
\end{tikzpicture}
        \caption{Before rewiring}
    \end{subfigure}
    \begin{subfigure}{0.2\textwidth}
        \centering
        \begin{tikzpicture}
\Vertex[x=0,y=0, label=$u$]{u}
\Vertex[x=0,y=1, label=$v$]{v}
\Vertex[x=1,y=0, label=$u'$]{uprime}
\Vertex[x=1,y=1, label=$v'$]{vprime}
\Edge[RGB,color={ 0.   , 114.432, 189.696}](u)(uprime)
\Edge[RGB,color={ 0.   , 114.432, 189.696}](v)(vprime)
\Vertex[x=-1,y=1.5, Pseudo]{dummy13}
\Vertex[x=-1,y=1, Pseudo]{dummy12}
\Vertex[x=-1,y=0, Pseudo]{dummy11}
\Edge[bend=25, opacity=0.2](u)(dummy11)
\Edge[bend=15, opacity=0.2](u)(dummy12)
\Edge[bend=-15, opacity=0.2](v)(dummy12)
\Edge[bend=-25, opacity=0.2](v)(dummy13)
\Vertex[x=2,y=1.5, Pseudo]{dummy23}
\Vertex[x=2,y=1, Pseudo]{dummy22}
\Vertex[x=2,y=0, Pseudo]{dummy21}
\Edge[bend=25, opacity=0.2](uprime)(dummy21)
\Edge[bend=-15, opacity=0.2](vprime)(dummy22)
\Edge[bend=25, opacity=0.2](vprime)(dummy23)
\end{tikzpicture}
        \caption{After rewiring}
    \end{subfigure}
    \vspace{-0.2cm}
    \caption{In the rewiring operation the red edges are deleted and the blue edges are added. The degree of each node remains the same.}
    \label{fig:rewiring}
\end{figure}

We will write $R_u$ to be the number of rewiring operations involving $u$ and write $\delta_u$ to be the smallest degree amongst either the nodes $u$ disconnects with or is now connected to. Each rewiring causes a single edge deletion and edge addition for each node involved so that the number of terms in each sum, the cardinality of sets $\mathcal{D}_u$ and $\mathcal{A}_u$, is $R_u$. Using this we can bound Eq.~(\ref{eq:Eu}) to get that
\begin{equation*}
\norm{\E_u^\intercal}_1 \leq \frac{1}{\sqrt{d_u+\gamma}} \left( \sum_{v \in \mathcal{D}_u} \frac{1}{\sqrt{\delta_u+\gamma}} + \sum_{v \in \mathcal{A}_u} \frac{1}{\sqrt{\delta_u+\gamma}} \right) = \frac{2R_u}{\sqrt{(d_u+\gamma)(\delta_u+\gamma)}}.
\end{equation*}

The largest possible value for the right hand side of the above equation over all nodes $u$ provides an upper bound for $\norm*{\E}_2$. From this, we can draw some conclusions as to when the filter will be robust to rewiring operations. The first is that one should not rewire around one node significantly as this will increase $R_u$. To keep $\norm*{\E}_1$ small, and thus $\norm*{\E}_2$ small, we must keep $\norm*{\E_u^\intercal}_1$ small for all nodes $u$ suggesting the perturbation should be distributed across the graph. Related to this observation, in \cite{polynomialstability} it was numerically demonstrated that the locality of perturbations can play a role in the magnitude of the error. The second strategy to ensure robustness is to rewire between high degree nodes. This will cause $d_u$ to be large and therefore $\norm*{\E_u^\intercal}_1$ will be small. Finally, using a normalised augmented adjacency matrix with larger values of $\gamma$ can cause smaller changes. 

\section{Stability of GCNN Models}
\label{sec:GNN}
We make use of results in previous sections to analyse the stability of two popular GCNN models, i.e., the simple graph convolutional networks (SGCN) \cite{SGCN} and the multilayered graph convolutional network (GCN) \cite{GCN}.

\subsection{SGCN}
The SGCN model is motivated by considering a multilayered GCN with the activation functions removed. By removing the activation functions the model boils down to a fixed monomial filter of order $K$ followed by applying a fully connected layer and a softmax layer to the node features.

Let the input data be given by the matrix $\X \in \R^{n \times d}$ where $d$ is the dimension of the features associated with each node. We can consider $\X$ as $d$ stacked graph signals which we will call feature maps. We will assume that each column (feature map) of $\X$ has unit norm. The output is a matrix $\mathbf{Y} \in \R^{n \times c}$,  representing class probabilities for each node.  The SGCN model is defined as $$\mathbf{Y} = \softmax(\tilde \GSO^K \X \mathbf{\Theta})$$
where $\softmax(\x)_i = \exp(\x_i)/\sum_i \exp(\x_i)$ normalises the rows to be probability distributions. We write $\tilde \GSO \equaldef \GSO_1$ as the normalised augmented adjacency matrix with $\gamma=1$. In this subsection, we will analyse how the logits, the node representations before the softmax is applied, change. The softmax function has a Lipschitz constant of $1$ so any bound on the logits can trivially be applied to the model outputs \cite{gao2017properties}[Proposition 4]. We begin by stating the following Lemma.
\begin{lemm}
\label{lemma:normineq}
Let $\mathbf{B} \in \R^{m \times r}$ and $\mathbf{C} \in \R^{r \times n}$ then 
\begin{equation*}
    \norm{\mathbf{BC}}_F \leq \norm{\mathbf{B}}_F \norm{\mathbf{C}}_2.
\end{equation*}
\end{lemm}
\begin{proof}
By decomposing $\B \C$ in terms of the rows $\mathbf{b}_k^\intercal$ of $\B$ we get 
\begin{equation}
\label{eq:lemma2eq}
    \norm{\B \C}^2_F = \sum\limits_{k=1}^m \norm{\mathbf{b}_k^\intercal \C}_2^2 \leq \sum\limits_{k=1}^m \norm{\mathbf{b}_k^\intercal}_2^2 \norm{\C}^2_2 = \norm{\B}_F^2 \norm{\C}_2^2.
\end{equation}
The inequality follows from the observation that $$\norm*{\mathbf{b}^\intercal \mathbf{C}}_2 = \norm*{(\mathbf{b}^\intercal \mathbf{C})^\intercal}_2 =\norm*{\mathbf{C}^\intercal\mathbf{b}}_2 \leq \norm*{\mathbf{C}^\intercal}_2\norm*{\mathbf{b}}_2=\norm*{\mathbf{b}}_2\norm*{\mathbf{C}}_2.$$ Taking the square root of both sides of Eq.~(\ref{eq:lemma2eq}) gives the result.
\end{proof}
Using this we can provide a bound on how much the logits can change under the Frobenius norm. To motivate why we are interested in the Frobenius norm here, consider taking the Euclidean norm of each node output to measure the amount of change in node representation. Taking the mean squared error of these distances amounts to taking the Frobenius norm of the logits matrix.
\begin{prop}
\label{prop:sgcnembed}
The distance of the logits is bounded like so
\begin{equation*}
     \norm{\tilde \GSO^K \X \mathbf{\Theta} - \tilde \GSO^K_p \X \mathbf{\Theta}}_F \leq \sqrt{d}K\norm{\E}_2\norm{\mathbf{\Theta}}_2 
\end{equation*}
\end{prop}
\begin{proof}
We first note that $\norm*{\X}_F = \sqrt{d}$. Using this, two applications of Lemma \ref{lemma:normineq}, and an application of Proposition \ref{prop:polystab} we get 
\begin{align*}
    \norm{\tilde \GSO^K \X \mathbf{\Theta} - \tilde \GSO^K_p \X \mathbf{\Theta}}_F &\leq \norm{\tilde \GSO^K \X - \tilde \GSO^K_p \X }_F \norm{\mathbf{\Theta}}_2 \\
    &\leq \norm{\tilde \GSO^K - \tilde \GSO^K_p}_2 \norm{\X}_F \norm{\mathbf{\Theta}}_2. \\
    &\leq \sqrt{d} K \norm{\E}_2 \norm{\mathbf{\Theta}}_2.  \qedhere
\end{align*}
\end{proof}

\subsection{Multilayer GCN}
We now consider a multilayered GCN model.
We again consider the logits of a model which this time consists of multiple GCN layers with pointwise non-linearities giving the $l$th layer representation as $$\X^{(l)} = \sigma(\tilde \GSO \X^{(l-1)} \mathbf{\Theta}^{(l)}),$$ where $\sigma$ is the ReLU activation function and $\mathbf{\Theta}^{(l)}$ are the layer parameters. We will consider the number of feature maps to be the same throughout the model.
\begin{prop}
\label{prop:gcnembed}
Let $\X^{(l)} = \sigma(\tilde \GSO \X^{(l-1)} \mathbf{\Theta}^{(l)})$ where $\X^{(0)}\in \R^{n \times d}$ is the input feature maps, $\mathbf{\Theta}^{(l)} \in \R^{d \times d}$ are the weight matrices and $L$ is the number of layers so $\X^{L}$ is the output features. Then 
\begin{equation*}
    \norm{\X^{(L)} - \X_p^{(L)}}_F \leq \sqrt{d}L\norm{\E}_2\prod_{l=1}^L \norm{\mathbf{\Theta}^{(l)}}_2. 
\end{equation*}
\end{prop}
\begin{proof}
We prove this by induction. The base case $L=1$ follows immediately from Proposition \ref{prop:sgcnembed}. Consider a more general $L$. Then
\begin{align*}
    \norm{\X^{(L)} - \X_p^{(L)}}_F &= \norm{\sigma(\tilde \GSO \X^{(l-1)} \mathbf{\Theta}^{(l)}) - \sigma(\tilde \GSO_p \X^{(l-1)}_p \mathbf{\Theta}^{(l)})}_F \\
    &\leq \norm{\tilde \GSO \X^{(l-1)} \mathbf{\Theta}^{(l)} - \tilde \GSO_p \X^{(l-1)}_p \mathbf{\Theta}^{(l)}}_F \\
    &\leq \norm{\tilde \GSO \X^{(l-1)}  - \tilde \GSO_p \X^{(l-1)}_p }_F \norm{\mathbf{\Theta}^{(l)}}_2, 
\end{align*}
where the first inequality comes from ReLU having unit Lipschitz constant and the second being an application of Lemma \ref{lemma:normineq}. By using triangle inequality we bound the following term
\begin{align*}
    \norm{\tilde \GSO \X^{(l-1)}  - \tilde \GSO_p \X^{(l-1)}_p }_F   
    &= \norm{\tilde \GSO \X^{(l-1)}  - \tilde \GSO_p \X^{(l-1)} + \tilde \GSO_p \X^{(l-1)} - \tilde \GSO_p \X^{(l-1)}_p }_F  \\
    &\leq \norm{\tilde \GSO \X^{(l-1)}  - \tilde \GSO_p \X^{(l-1)}}_F + \norm{\tilde \GSO_p \X^{(l-1)} - \tilde \GSO_p \X^{(l-1)}_p }_F  \\
    &\leq \norm{\E}_2\norm{ \X^{(l-1)}}_F + \norm{\X^{(l-1)}  - \X_p^{(l-1)}}_F  .
\end{align*}
Note that $$\norm*{\X^{(l)}}_F = \norm*{\sigma(\tilde \GSO \X^{(l-1)} \mathbf{\Theta}^{(l)})}_F \leq \norm*{\X^{(l-1)}}_F\norm*{\mathbf{\Theta}^{(l)}}_2,$$ so by recursivity we get that $$\norm*{\X^{(l)}}_F \leq \norm*{\X^{(0)}}_F \norm*{\mathbf{\Theta}^{(1)}}_2 \ldots  \norm*{\mathbf{\Theta}^{(l)}}_2.$$ Recall that $\norm*{\X^{(0)}}_F = \sqrt{d}$. Using this observation and the inductive assumption we get that 
\begin{align*}
\left(\norm{\E}_2\norm{ \X^{(l-1)}}_F + \norm{\X^{(l-1)}  - \X_p^{(l-1)}}_F  \right) \norm{\mathbf{\Theta}^{(l)}}_2 
&\leq \sqrt{d} \norm{\E}_2 \prod_{l=1}^L \norm{\mathbf{\Theta}^{(l)}}_2 + \sqrt{d}(L-1)\norm{\E}_2 \prod_{l=1}^L \norm{\mathbf{\Theta}^{(l)}}_2\\
&=\sqrt{d}L\norm{\E}_2\prod_{l=1}^L \norm{\mathbf{\Theta}^{(l)}}_2. \qedhere
\end{align*}
\end{proof}
We finish this section by combining Proposition \ref{prop:gcnembed} with the results from Sec.~\ref{sec:rewiring} to give the following.
\begin{corr}
Consider the GCN outputs $\X^{(L)}$ and $\X_p^{(L)}$ for a graph $\G$ and a perturbed graph $\G_p$ where the perturbed graph is a result of double edge rewiring. Let $\tilde \GSO$ and $\tilde \GSO_p$ be the corresponding normalised augmented adjacency matrices. Define $d_u$, $\delta_u$ and $R_u$ as in Sec.~\ref{sec:rewiring}, then the following holds 
\begin{equation*}
    \norm{\X^{(L)} - \X_p^{(L)}}_F \leq \sqrt{d}L\prod_{l=1}^L\norm{\mathbf{\Theta}^{(l)}}_2 \max_{u \in \V}\frac{2R_u}{\sqrt{(d_u+1)(\delta_u+1)}}.
\end{equation*}
\end{corr}
A similar result holds for SGCN by combining Proposition \ref{prop:sgcnembed} with the results from Sec.~\ref{sec:rewiring}. We suspect these bounds will likely be loose in practice; nevertheless, they provide conceptional insights into the factors that may be related to the robustness of the GCN and SGCN models. In particular, we can reason when this bound will be small as in the final paragraph of Sec.~\ref{sec:rewiring}, which relates interpretable structural perturbation to the robustness of these models. 

\section{Discussion}
In this work, we bound the change in output of spectral graph filters under a specific form of topological perturbation, i.e., edge rewiring, where the bound involves terms which have a structural interpretation. 
We then demonstrate a practical application of this bound by applying it to the SGCN and GCN models. Future directions include proving that the change in output of other models scales proportional to $\norm*{\E}_2$ and providing further interpretable bounds to $\norm*{\E}_2$ to provide an understanding of stability for graph-based models. Furthermore, the SGCN and GCN models we consider both apply low-pass filters, and exploring models with other filtering characteristics is an interesting future direction. Finally, extensions of the framework presented in this work such as considering a more general perturbation model or extending it to weighted graphs will benefit its utility for practical applications.  
\bibliographystyle{IEEEbib}
\bibliography{bibliography}
\end{document}